\documentclass[letterpaper,english]{article}

\usepackage{ijcai15}
\usepackage{times}

\usepackage[utf8]{inputenc}
\usepackage[T1]{fontenc}
\usepackage{babel}

\usepackage{amsmath}
\usepackage{amsfonts}
\usepackage{amssymb}
\usepackage{amsthm}
\usepackage{mathtools}
\usepackage{stmaryrd}
\usepackage{graphicx}
\usepackage{booktabs}
\usepackage{enumitem}
\usepackage{xspace}
\usepackage[obeyDraft]{todonotes}
\usepackage{macros}

\usepackage{url}

\theoremstyle{plain}
\newtheorem{theorem}{Theorem}
\newtheorem{lemma}[theorem]{Lemma}

\theoremstyle{definition}

\DeclareSymbolFont{upshape}{OT1}{cmr}{m}{n}
\DeclareMathSymbol{\Gamma}{\mathalpha}{upshape}{0}
\DeclareMathSymbol{\Delta}{\mathalpha}{upshape}{1}

\newcommand{\ie}{i.e.\ }
\newcommand{\wrt}{w.r.t.\ }

\renewcommand{\epsilon}{\varepsilon}
\newcommand{\eps}{\epsilon}
\renewcommand{\le}{\leqslant}

\renewcommand{\ge}{\geqslant}

\newcommand{\axiom}[1]{\ensuremath{\langle #1\rangle}\xspace}
\newcommand{\assert}[2]{\ensuremath{#1\hspace{0.08em}{:}\hspace{0.05em}#2}\xspace}
\newcommand{\Go}{\ensuremath{\mathsf{G}}\xspace}
\newcommand{\ra}{\ensuremath{\Rightarrow}\xspace}
\newcommand{\atLeast}[3]{\ensuremath{{\ge}#1\,#2.#3}\xspace}
\newcommand{\atMost}[3]{\ensuremath{{\le}#1\,#2.#3}\xspace}

\newcommand{\ls}{\ensuremath{\preccurlyeq}}

\newcommand{\red}{\ensuremath{\mathsf{red}}\xspace}
\newcommand{\inv}{\ensuremath{\mathsf{inv}}\xspace}
\newcommand{\rol}{\ensuremath{\mathsf{rol}}\xspace}

\newcommand{\fuzzy}{\fsf}
\newcommand{\crisp}{\csf}
\newcommand{\values}{\ensuremath{\Vmc_\Omc}\xspace}
\newcommand{\parent}[1]{\ensuremath{#1_\uparrow}\xspace}
\newcommand{\up}[1]{\ensuremath{\langle#1\rangle_\uparrow}\xspace}
\renewcommand{\ALCQ}{\ensuremath{\ALC\kern-0.16em\Qmc}\xspace}
\newcommand{\IALC}{\ensuremath{\mathfrak{I}\kern-0.14em\ALC}\xspace}
\newcommand{\IALCQ}{\ensuremath{\mathfrak{I}\kern-0.14em\ALCQ}\xspace}
\newcommand{\SHIQ}{\ensuremath{\Smc\kern-0.05em\Hmc\Imc\kern-0.18em\Qmc}\xspace}
\newcommand{\SRIQ}{\ensuremath{\Smc\kern-0.05em\Rmc\kern-0.03em\Imc\kern-0.18em\Qmc}\xspace}
\newcommand{\SROQ}{\ensuremath{\Smc\kern-0.05em\Rmc\kern-0.05em\Omc\kern-0.14em\Qmc}\xspace}
\newcommand{\SHOI}{\ensuremath{\Smc\kern-0.05em\Hmc\kern-0.1em\Omc\kern-0.03em\Imc}\xspace}
\newcommand{\SROI}{\ensuremath{\Smc\kern-0.05em\Rmc\kern-0.05em\Omc\kern-0.03em\Imc}\xspace}
\newcommand{\SROIQ}{\ensuremath{\Smc\kern-0.05em\Rmc\kern-0.05em\Omc\kern-0.03em\Imc\kern-0.18em\Qmc}\xspace}
\newcommand{\ALCIQ}{\ensuremath{\ALC\kern-0.03em\Imc\kern-0.18em\Qmc}\xspace}
\newcommand{\ALCOIQ}{\ensuremath{\ALC\kern-0.16em\Omc\kern-0.03em\Imc\kern-0.18em\Qmc}\xspace}

\renewcommand{\path}{\varrho}
\newcommand{\tail}{\ensuremath{\mathsf{tail}}\xspace}
\newcommand{\prev}{\ensuremath{\mathsf{prev}}\xspace}
\newcommand{\role}{\ensuremath{\mathsf{role}}\xspace}

\newcommand{\oc}[1]{\scalebox{0.93}{\ensuremath{\fbox{$#1$}}}\xspace}
\title{Reasoning in Infinitely Valued \Go-\IALCQ}
\author{Stefan Borgwardt \\
	Theoretical Computer Science, \\ TU Dresden, Germany \\
	\url{Stefan.Borgwardt@tu-dresden.de}
\And
	Rafael Pe{\~n}aloza \\
	KRDB Research Centre, \\ Free University of Bozen-Bolzano, Italy \\
	\url{rafael.penaloza@unibz.it}
}
\begin{document}

\maketitle

\begin{abstract}
%\todo[inline]{more AI}
  Fuzzy Description Logics (FDLs) are logic-based formalisms used to represent and reason with vague
  or imprecise knowledge.
  It has been recently shown that reasoning in most FDLs using truth values from the interval
  $[0,1]$ becomes undecidable in the presence of a negation
  constructor and general concept inclusion axioms.
  One exception to this negative result are FDLs whose semantics is based on the infinitely valued
  Gödel t-norm~(\Go).
  In this paper, we extend previous decidability results for \Go-\IALC to deal also with qualified
  number restrictions.
  Our novel approach is based on a combination of the known crispification
  technique for finitely valued FDLs and the automata-based procedure originally developed for
  reasoning in \mbox{\Go-\IALC}.
   The proposed approach combines the advantages of these two methods, while removing their
   respective drawbacks.
\end{abstract}

\section{Introduction}

It is well-known that one of the main requirements for the development of an intelligent application is a formalism
capable of representing and handling knowledge without ambiguity.
Description Logics (DLs) are a well-studied family of knowledge representation formalisms~\cite{BCM+07}.
They constitute the logical backbone of the standard Semantic Web ontology language OWL\,2,%
\footnote{\url{http://www.w3.org/TR/owl2-overview/}} and its profiles,
and have been successfully applied to represent the knowledge of many and diverse application domains,
particularly in the bio-medical sciences.

DLs describe the domain knowledge using \emph{concepts} (such as
\textsf{Patient}) that represent sets of individuals, and \emph{roles}
(\textsf{hasRelative}) that represent connections between individuals.
\emph{Ontologies} are collections of axioms formulated over these concepts and
roles, which restrict their possible interpretations. 
The typical axioms considered in
DLs are \emph{assertions}, like $\assert{\textsf{bob}}{\textsf{Patient}}$, providing knowledge about
specific individuals; 
and \emph{general concept inclusions (GCIs)}, such as
$\textsf{Patient}\sqsubseteq\textsf{Human}$, which express general relations between concepts.
Different DLs are characterized by the constructors allowed to generate
complex concepts and roles from atomic ones.
\ALC~\cite{ScSm-AI91} is a prototypical DL of intermediate expressivity that
contains the  concept constructors conjunction ($C\sqcap D$), negation
($\lnot C$), and existential restriction ($\exists r.C$ for a role~$r$).
If additionally qualified number restrictions ($\atLeast{n}{r}{C}$ for
$n\in\naturals$) are allowed, the resulting logic is denoted by \ALCQ.
Common reasoning problems in \ALCQ, such as consistency of ontologies or
subsumption between concepts, are known to be
\ExpTime-complete~\cite{Schi-IJCAI91,Tobi-01}.

Fuzzy Description Logics (FDLs) have been introduced as extensions of classical DLs
to represent and reason with vague knowledge. The main idea is to consider
all the truth values from the interval $[0,1]$
instead of only \emph{true} and \emph{false}. In this way, it is possible give a more
fine-grained semantics to inherently vague concepts like $\textsf{LowFrequency}$ or
$\textsf{HighConcentration}$, which can be found in biomedical ontologies like
{\sc Snomed\,CT},%
\footnote{\url{http://www.ihtsdo.org/snomed-ct/}}
and Galen.%
\footnote{\url{http://www.opengalen.org/}}
The different members of the family of FDLs are characterized not only by the constructors they 
allow, but also by the way these constructors are interpreted.

To interpret conjunction in complex concepts like
\begin{align*}
\exists\textsf{hasHeartRate}.\textsf{LowFrequency}\sqcap {} \qquad \qquad \qquad \\
  \exists\textsf{hasBloodAlcohol}.\textsf{HighConcentration},
\end{align*}
a popular
approach is to use so-called \emph{t-norms}~\cite{KlMP-00}.
The semantics of the other logical constructors can then be
derived from these t-norms in a principled way, as suggested by~\citeauthor{Haje-01}~\shortcite{Haje-01}.
Following the principles of mathematical fuzzy logic, existential
restrictions are interpreted as suprema
of truth values. However, to avoid problems with infinitely many truth
values, reasoning in fuzzy DLs is often restricted to so-called \emph{witnessed
models}~\cite{Haje-FSS05}, in which these suprema are required to be maxima;
i.e., the degree is witnessed by at least one domain element.

Unfortunately, reasoning in most FDLs becomes undecidable when the logic allows to use
GCIs and negation under witnessed model
semantics~\cite{BaPe-FroCoS11,CeSt-ISc13,BoDP-AI15}.
One of the few exceptions known are FDLs using the \emph{Gödel} t-norm 
defined as $\min\{x,y\}$
to interpret conjunctions~\cite{BoDP-KR14}.
Despite not being as well-behaved as finitely valued FDLs, which use a finite
total order of truth values instead of the infinite interval $[0,1]$~\cite{BoPe-JoDS13}, it has been
shown using an automata-based approach that reasoning in Gödel extensions of
\ALC exhibits the same complexity as in the classical case, \ie it is
\ExpTime-complete.
A major drawback of this approach is that it always has an exponential runtime,
even when the input ontology has a simple form.

In this paper, we extend the results of~\cite{BoDP-KR14} to deal with
qualified number restrictions, showing again that the complexity of reasoning remains the 
same as for the classical case; i.e., it is \ExpTime-complete.
To this end, we focus only on the problem of \emph{local
consistency}, which is a generalization of the classical concept satisfiability
problem.
We follow a more practical approach that combines the automata-based
construction from~\cite{BoDP-KR14} with reduction techniques developed for
finitely valued FDLs~\cite{Stra-JELIA04,BDGS-IJAR09,BoSt-13}.
We exploit the forest model property of classical \ALCQ~\cite{Kaza-JELIA04} to
encode order relationships between concepts in a fuzzy interpretation in a
manner similar to the Hintikka trees from~\cite{BoDP-KR14}.
However, instead of using automata to determine the existence of such trees, we
reduce the fuzzy ontology directly into a classical \ALCQ ontology whose
local consistency is equivalent to that of the original ontology.
This enables us to use optimized reasoners for classical DLs.
In addition to the \emph{cut-concepts} of the form $\oc{C\ge q}$ for a fuzzy
concept~$C$ and a value~$q$, which are used in the reductions for finitely
valued DLs~\cite{Stra-JELIA04,BDGS-IJAR09,BoSt-13}, we employ \emph{order
concepts} $\oc{C\le D}$ expressing relationships between fuzzy concepts.
Contrary to the reductions for finitely valued Gödel
FDLs presented 
by~\citeauthor{BDGS-IJAR09}~\shortcite{BDGS-IJAR09,BDGS-IJUF12}, 
our reduction produces a classical ontology whose size is polynomial in the 
size of the input
fuzzy ontology. Thus, we obtain tight complexity bounds for 
reasoning in this FDL~\cite{Tobi-01}.
An extended version of this paper appears in~\cite{BoPe-FroCoS15}.

%\todo[inline]{mention that this has been retracted from the DL workshop (in
%case we get the same reviewers)?}
%\todo[inline]{Should we explain also the sudden disappearance of the result for
%consistency?}

\section{Preliminaries}

For the rest of this paper, we focus solely on vague statements that take truth
degrees from the infinite interval $[0,1]$, where the \emph{Gödel t-norm}, defined by $\min\{x,y\}$,
is used to interpret logical conjunction. The semantics of implications is given by the
\emph{residuum} of this t-norm; that is,
\[ x\ra y:=\begin{cases}
  1 &\text{if $x\le y$,} \\
  y &\text{otherwise.}
\end{cases} \]
We use both the \emph{residual negation} $\ominus x:=x\ra 0$ and the
\emph{involutive negation} ${\sim}x:=1-x$ in the rest of this paper.

We first recall some basic definitions from~\cite{BoDP-KR14}, which will be
used extensively in the proofs throughout this work.
An \emph{order structure} $S$ is a finite set containing at least the
numbers~$0$, $0.5$, and~$1$, together with an involutive unary operation
$\inv\colon S\to S$ such that $\inv(x)=1-x$ for all $x\in S\cap[0,1]$.
A \emph{total preorder} over~$S$ is a transitive and total binary relation
${\ls}\subseteq S\times S$. For $x,y\in S$, we write $x\equiv y$ if
$x\ls y$ and $y\ls x$.
Notice that $\equiv$ is an equivalence relation on~$S$.
The total preorders considered in~\cite{BoDP-KR14} have to satisfy additional
properties; for instance, that $0$ and $1$ are always the least and greatest elements,
respectively.
These properties can be found in our reduction in the axioms of $\red(\Umc)$
(see Section~\ref{sec:local-consistency} for more details).

The syntax of the FDL \Go-\IALCQ is the same as that of classical \ALCQ, with
the addition of the implication constructor (denoted by the use of~\Imf at the beginning of the name).
This constructor is often added to FDLs, as the residuum cannot, in general, be
expressed using only the t\mbox{-}norm and negation operators, in contrast to the classical 
semantics. In particular, this holds for the G\"odel t-norm and its residuum, which is the focus of this
work.
Let now \NC, \NR, and \NI be mutually disjoint sets of \emph{concept},
\emph{role}, and \emph{individual names}, respectively. \emph{Concepts} of
\Go-\IALCQ are built using the syntax rule
\[ C,D::=\top\mid A\mid \lnot C\mid C\sqcap D\mid C\to D\mid\forall r.C\mid
  \atLeast{n}{r}{C}, \]
where $A\in\NC$, $r\in\NR$, $C,D$ are concepts, and $n\in\naturals$.
We use the abbreviations 
\begin{align*}
\bot:={} & \lnot\top, \\
C\sqcup D:= {} & \lnot(\lnot C\sqcap\lnot D), \\
\exists r.C:= {} & \atLeast{1}{r}{C}, \text{\quad \quad \quad and} \\
\atMost{n}{r}{C}:= {} & \lnot(\atLeast{(n+1)}{r}{C})
\end{align*}
Notice that~\citeauthor{BDGS-IJUF12}\ consider a different definition 
of \emph{at-most restrictions}, which
uses the residual negation; that is, they define
$\atMost{n}{r}{C}:=(\atLeast{(n+1)}{r}{C})\to\bot$~\shortcite{BDGS-IJUF12}.
This has the strange side effect that the value of $\atMost{n}{r}{C}$ is always
either~$0$ or~$1$ (see the semantics below).
However, this discrepancy in definitions is not an issue since our algorithm
can handle both cases.

The semantics of this logic is based on interpretations. A \emph{\Go-interpretation} is a pair
$\Imc=(\Delta^\Imc,\cdot^\Imc)$, where $\Delta^\Imc$ is a non-empty set called the \emph{domain}, and
$\cdot^\Imc$ is the \emph{interpretation function} that assigns to each individual name $a\in\NI$
an element $a^\Imc\in\Delta^\Imc$, to each concept name $A\in\NC$ a fuzzy set
$A^\Imc\colon\Delta^\Imc\to[0,1]$, and to each role name $r\in\NR$ a fuzzy
binary relation $r^\Imc\colon\Delta^\Imc\times\Delta^\Imc\to[0,1]$.
%\footnote{When the semantics is restricted to the values $0$ and $1$, we obtain
%structures that are isomorphic to classical interpretations.}
%
The interpretation of complex concepts is obtained from the semantics of
first-order fuzzy logics via the well-known translation from DL
concepts to first-order logic~\cite{Stra-JAIR01,BDGS-IJUF12}, \ie for all
$d\in\Delta^\Imc$,
\begin{align*}
  \top^\Imc(d) &:= 1 \\
  (\lnot C)^\Imc(d) &:= 1-C^\Imc(d) \\
  (C\sqcap D)^\Imc(d) &:= \min\{C^\Imc(d),D^\Imc(d)\} \\
  (C\to D)^\Imc(d) &:= C^\Imc(d)\ra D^\Imc(d) \\
%  (\exists r.C)^\Imc(d) &:= 
%\sup_{e\in\Delta^\Imc}\min\{r^\Imc(d,e),C^\Imc(e)\}
%    \\
  (\forall r.C)^\Imc(d) &:= \inf_{e\in\Delta^\Imc}r^\Imc(d,e)\ra C^\Imc(e) \\
  (\atLeast{n}{r}{C})^\Imc(d) &:=
    \sup_{\substack{e_1,\dots,e_n\in\Delta^\Imc\\\text{pairwise different}}}
      \min_{i=1}^n\min\{r^\Imc(d,e_i),C^\Imc(e_i)\}
%  (\atMost{n}{r}{C})^\Imc(d) &:=
%    \inf_{\substack{e_1,\dots,e_{n+1}\in\Delta^\Imc\\
%      \text{pairwise different}}}
%      \ominus\min_{i=1}^{n+1}\min\{r^\Imc(d,e_i),C^\Imc(e_i)\}
\end{align*}
Recall that
the usual duality between existential and value restrictions that appears in classical DLs 
does not hold in \mbox{\Go-\IALCQ}.

A \emph{classical interpretation} is defined similarly, with the set of truth
values restricted to~$0$ and~$1$. In this case, the semantics of a concept~$C$
is commonly viewed as a set $C^\Imc\subseteq\Delta^\Imc$ instead of the
characteristic function $C^\Imc\colon\Delta^\Imc\to\{0,1\}$.

%Note that the residual negation $\ominus x$ is $1$ if $x=0$, and $0$ whenever
%$x>0$. Thus, the semantics of at-most restrictions can be characterized as
%follows.
%\begin{proposition}
%\label{prop:at-most-semantics}
%  For all interpretations~\Imc, at-most restrictions $\atMost{n}{r}{C}$, and
%  elements $d\in\Delta^\Imc$, we have $(\atMost{n}{r}{C})^\Imc(d)\in\{0,1\}$.
%%
%  Furthermore, $(\atMost{n}{r}{C})^\Imc(d)=1$ iff there are at most $n$
%  elements $e\in\Delta^\Imc$ such that $\min\{r^\Imc(d,e),C^\Imc(e)\}>0$.
%\end{proposition}
%
In the following, we restrict all reasoning problems to so-called
\emph{witnessed} \Go-interpretations~\cite{Haje-FSS05}, which intuitively
require the suprema and infima in the semantics to be maxima and minima,
respectively.
More formally, the \mbox{\Go-interpretation}~\Imc is \emph{witnessed} if, for every
$d\in\Delta^\Imc$, $n\ge 0$, $r\in\NR$, and concept~$C$, there exist
$e,e_1,\dots,e_n\in\Delta^\Imc$ (where $e_1,\dots,e_n$ are pairwise
different) such that
\begin{align*}
%  (\exists r.C)^\Imc(d) &= \min\{r^\Imc(d,e),C^\Imc(e)\}, \\
  (\forall r.C)^\Imc(d) &= r^\Imc(d,e)\ra C^\Imc(e)\text{ and} \\
  (\atLeast{n}{r}{C})^\Imc(d) &= \min_{i=1}^n\min\{r^\Imc(d,e_i),C^\Imc(e_i)\}.
\end{align*}

The axioms of \Go-\IALCQ extend classical axioms by allowing to compare degrees
of arbitrary assertions in so-called \emph{ordered ABoxes}~\cite{BoDP-KR14},
and to state inclusions relationships between fuzzy concepts that hold to a
certain degree, instead of only~$1$.
A \emph{classical assertion} is an expression of the form $\assert{a}{C}$ or
$\assert{(a,b)}{r}$ for $a,b\in\NI$, $r\in\NR$, and a concept~$C$. An
\emph{order assertion} is of the form $\axiom{\alpha\bowtie q}$ or
$\axiom{\alpha\bowtie\beta}$ where ${\bowtie}\in\{<,\le,=,\ge,>\}$,
$\alpha,\beta$ are classical assertions, and $q\in[0,1]$.
A \emph{(fuzzy) general concept inclusion axiom (GCI)} is of the form
$\axiom{C\sqsubseteq D\ge q}$ for concepts $C,D$ and $q\in[0,1]$.
An \emph{ordered ABox} is a finite set of order assertions, a \emph{TBox} is a
finite set of GCIs, and an \emph{ontology} $\Omc=(\Amc,\Tmc)$ consists of an
ordered ABox~\Amc and a TBox~\Tmc.
A \Go-interpretation~\Imc \emph{satisfies} (or is a \emph{model} of) an order
assertion
$\axiom{\alpha\bowtie\beta}$ if $\alpha^\Imc\bowtie\beta^\Imc$ (where
$(\assert{a}{C})^\Imc:=C^\Imc(a^\Imc)$,
$(\assert{(a,b)}{r})^\Imc:=r^\Imc(a^\Imc,b^\Imc)$, and $q^\Imc:=q$);
it \emph{satisfies} a GCI $\axiom{C\sqsubseteq D\ge q}$ if
$C^\Imc(d)\ra D^\Imc(d)\ge q$ holds for all $d\in\Delta^\Imc$; and
it \emph{satisfies} an ordered ABox, TBox, or ontology if it satisfies all its
axioms.
An ontology is \emph{consistent} if it has a (witnessed) model.

Given an ontology~\Omc, we denote by $\rol(\Omc)$ the set of all role names
occurring in~\Omc and by $\sub(\Omc)$ the closure under negation of the set of
all subconcepts occurring in~\Omc. We consider the concepts $\lnot\lnot C$ and 
$C$ to be equal,
and thus the latter set is of quadratic size in the size of~\Omc.
Moreover, we denote by \values the closure under the involutive negation
$x\mapsto 1-x$ of the set of all truth degrees appearing in~\Omc, together
with~$0$, $0.5$, and~$1$. This set is of size linear on the size of~\Omc.
We sometimes denote the elements of $\values\subseteq[0,1]$ as
$0=q_0<q_1<\dots<q_{k-1}<q_k=1$.

We stress that we do not consider the general consistency problem in this
paper, but only a restricted version that uses only one individual name. 
An ordered ABox~\Amc is \emph{local} if it contains no role assertions
$\assert{(a,b)}{r}$ and there is a single individual name $a\in\NI$ such that
all order assertions in~\Amc only use~$a$.
The \emph{local consistency} problem, \ie deciding whether an ontology
$(\Amc,\Tmc)$ with a local ordered ABox~\Amc is consistent, can be seen as a
generalization of the classical concept satisfiability
problem~\cite{BoPe-JoDS13}.

Other common reasoning problems for FDLs, such as concept satisfiability and
subsumption can be reduced to local consistency~\cite{BoDP-KR14}: the
subsumption between $C$ and $D$ to degree~$q$ \wrt a TBox~\Tmc is equivalent to
the (local) inconsistency of $(\{\axiom{\assert{a}{C\to D}<q}\},\Tmc)$, and the
satisfiability of~$C$ to degree~$q$ w.r.t.~\Tmc is equivalent to the (local)
consistency of $(\{\axiom{\assert{a}{C}\ge q}\},\Tmc)$.

In the following section we show how to decide local consistency of a \Go-\IALCQ ontology through
a reduction to classical ontology consistency.

\section{Deciding Local Consistency}
\label{sec:local-consistency}

Let $\Omc=(\Amc,\Tmc)$ be a \Go-\IALCQ ontology where \Amc is a local ordered ABox
that uses only the individual name~$a$.
The main ideas behind the reduction to classical \ALCQ are that it suffices to
consider tree-shaped interpretations, where each domain element has a unique
role predecessor, and that we only have to consider the \emph{order} between
values of concepts, instead of their precise values.
This insight allows us to consider only finitely many different cases~\cite{BoDP-KR14}.

To compare the values of the elements of $\sub(\Omc)$ at different domain
elements, we  use the order structure
\[ \Umc :=
  \values\cup\sub(\Omc)\cup\sub_\uparrow(\Omc)\cup\{\lambda,\lnot\lambda\}, \]
where $\sub_\uparrow(\Omc):=\{\up{C} \mid C\in\sub(\Omc)\}$,
$\inv(\lambda):=\lnot\lambda$, $\inv(C):=\lnot C$, and
$\inv(\up{C}):=\up{\lnot C}$, for all concepts $C\in\sub(\Omc)$.
The idea is that total preorders over~\Umc describe the relationships between
the values of $\sub(\Omc)$ and~\values at a single domain element. The elements
of $\sub_\uparrow(\Omc)$ allow us to additionally refer to the relevant values
at the unique role predecessor of the current domain element (in a tree-shaped
interpretation). The value~$\lambda$ represents the value of the role
connection from this predecessor.
For convenience, we define $\up{q}:=q$ for all $q\in\values$.

In order to describe such total preorders in a classical \ALCQ ontology, we
employ special concept names of the form $\oc{\alpha\le\beta}$ for
$\alpha,\beta\in\Umc$.
This differs from previous reductions for finitely
valued FDLs~\cite{Stra-JELIA04,BoSt-ISC11,BDGS-IJUF12} in that we not only
consider \emph{cut-concepts} of the form $\oc{q\le\alpha}$ with $q\in\values$,
but also relationships between different concepts.%
\footnote{For the rest of this paper, expressions of the form \oc{\alpha\le\beta}
denote (classical) concept names.}
For convenience, we introduce the abbreviations
$\oc{\alpha\ge\beta}:=\oc{\beta\le\alpha}$,
$\oc{\alpha<\beta}:=\lnot\oc{\alpha\ge\beta}$, and similarly for~$=$ and $>$.
Furthermore, we define the complex expressions
\begin{itemize}
  \item $\oc{\alpha\ge\min\{\beta,\gamma\}} :=
    \oc{\alpha\ge\beta}\sqcup\oc{\alpha\ge\gamma}$,
  \item $\oc{\alpha\le\min\{\beta,\gamma\}} :=
    \oc{\alpha\le\beta}\sqcap\oc{\alpha\le\gamma}$,
  \item $\oc{\alpha\ge\beta\ra\gamma} :=
    (\oc{\beta\le\gamma}\to\oc{\alpha\ge 1})\sqcap
    (\oc{\beta>\gamma}\to\oc{\alpha\ge\gamma})$,
  \item $\oc{\alpha\le\beta\ra\gamma} :=
    \oc{\beta\le\gamma}\sqcup\oc{\alpha\le\gamma}$,
\end{itemize}
and extend these notions to the expressions $\oc{\alpha\bowtie\beta\ra\gamma}$ etc., for
${\bowtie}\in\{<,=,>\}$, analogously.

For each concept $C\in\sub(\Omc)$, we now define the classical \ALCQ TBox
$\red(C)$, depending on the form of $C$, as follows.
\begin{align*}
  \red(\top) :={}
    & \{\top\sqsubseteq\oc{\top\ge 1}\} \\
  \red(\lnot C) := {}
    & \emptyset \\
  \red(C\sqcap D) :={}
    & \{\top\sqsubseteq\oc{C\sqcap D=\min\{C,D\}}\} \\
  \red(C\to D) :={}
    & \{\top\sqsubseteq\oc{C\to D=C\ra D}\} \displaybreak[0] \\
%  \red(\exists r.C) :={}
%    & \{\top\sqsubseteq
%      \exists r.\oc{\up{\exists r.C}\le\min\{\lambda,C\}}\sqcap
%      \forall r.\oc{\up{\exists r.C}\ge\min\{\lambda,C\}}\} \\
  \red(\forall r.C) :={}
    & \{\top\sqsubseteq
      \exists r.\oc{\up{\forall r.C}\ge\lambda\ra C}\sqcap {} \\
    & \phantom{\{\top\sqsubseteq{}}
      \forall r.\oc{\up{\forall r.C}\le\lambda\ra C}\} \displaybreak[0] \\
  \red(\atLeast{n}{r}{C}) :={}
    & \{\top\sqsubseteq
      \atLeast{n}{r}{\oc{\up{\atLeast{n}{r}{C}}\le\min\{\lambda,C\}}}\sqcap{} \\
    & \phantom{\{\top\sqsubseteq{}}
      \lnot\atLeast{n}{r}{\oc{\up{\atLeast{n}{r}{C}}<\min\{\lambda,C\}}}\}
%  \red(\atMost{n}{r}{C}) :={}
%    & \{\top\sqsubseteq\oc{1\le(\atMost{n}{r}{C})}\sqcup
%      \oc{(\atMost{n}{r}{C})\le 0}, \\
%    & \phantom{\{}
%      
%\oc{1\le(\atMost{n}{r}{C})}\equiv\atMost{n}{r}{\oc{0<\min\{\lambda,C\}}}\}
\end{align*}
Intuitively, $\red(C)$ describes the semantics of~$C$ in terms of its order
relationships to other elements of~\Umc.
Note that the semantics of the involutive negation $\lnot C=\inv(C)$ is already
handled by the operator~\inv (see also the last line of the definition of $\red(\Umc)$ below).

The reduced classical \ALCQ ontology $\red(\Omc)$ is defined as follows:
\begin{align*}
  \red(\Omc) :={} 
    & (\red(\Amc),\red(\Umc)\cup\red(\uparrow)\cup\red(\Tmc)), \\
  \red(\Amc) :={}
    & \{\assert{a}{\oc{C\bowtie q}}\mid\axiom{\assert{a}{C}\bowtie 
    q}\in\Amc\}\cup {} \\
    & \{\assert{a}{\oc{C\bowtie D}}\mid
        \axiom{\assert{a}{C}\bowtie\assert{a}{D}}\in\Amc\}, \displaybreak[0] \\
  \red(\Umc) :={}
    & \{\oc{\alpha\le\beta}\sqcap\oc{\beta\le\gamma}\sqsubseteq
        \oc{\alpha\le\gamma} \mid
      \alpha,\beta,\gamma\in\Umc\} \cup{} \\
    & \{\top\sqsubseteq\oc{\alpha\le\beta}\sqcup\oc{\beta\le\alpha} \mid
      \alpha,\beta\in\Umc\} \cup{} \\
    & \{\top\sqsubseteq\oc{0\le\alpha}\sqcap\oc{\alpha\le 1} \mid
      \alpha\in\Umc\} \cup{} \\
    & \{\top\sqsubseteq\oc{\alpha\bowtie\beta} \mid
      \alpha,\beta\in\values,\ \alpha\bowtie\beta\} \cup{} \\
    & \{\oc{\alpha\le\beta}\sqsubseteq\oc{\inv(\beta)\le\inv(\alpha)} \mid
      \alpha,\beta\in\Umc\}, \displaybreak[0] \\
  \red(\uparrow) := {}
    & \{\oc{\alpha\bowtie\beta}\sqsubseteq
      \forall r.\oc{\up{\alpha}\bowtie\up{\beta}} \mid{} \\
    &\quad\quad\quad\quad
      \alpha,\beta\in\values\cup\sub(\Omc),\ r\in\rol(\Omc)\}, \\
  \red(\Tmc) :={}
    & \{\top\sqsubseteq\oc{q\le C\ra D} \mid
      \axiom{C\sqsubseteq D\ge q}\in\Tmc\} \cup {} \\
    & \bigcup_{C\in\sub(\Omc)}\red(C).
\end{align*}
We briefly explain this construction.
The reductions of the order assertions and fuzzy GCIs in~\Omc are
straightforward; the former expresses that the individual $a$ must belong to the corresponding
order concept $C\bowtie q$ or $C\bowtie D$, while the latter expresses that 
every element of the domain must satisfy 
the restriction provided by the fuzzy GCI. The axioms of $\red(\Umc)$ intuitively ensure that the relation ``$\le$''
forms a total preorder that is compatible with all the values in~\values, and that
$\inv$ is an antitone operator.
Finally, the TBox $\red(\uparrow)$ expresses a connection between the orders of a domain
element and those of its role successors.

%\todo[inline]{I don't know whether an example would make sense here?}

The following lemmata show that this reduction is correct; i.e., that it preserves local consistency.

\begin{lemma}
\label{lem:soundness}
  If $\red(\Omc)$ has a classical model, then \Omc has a \Go-model.
\end{lemma}
\begin{proof}
  By~\cite{Kaza-JELIA04}, $\red(\Omc)$ must have a \emph{tree model}~\Imc, \ie
  we can assume that $\Delta^\Imc$ is a prefix-closed subset of $\naturals^*$,
  $a^\Imc=\eps$, for all $n_1,\dots,n_k\in\naturals$, $k\ge 1$, with
  $u:=n_1\dots n_k\in\Delta^\Imc$, the element
  $\parent{u}:=n_1\dots n_{k-1}\in\Delta^\Imc$ is an $r$-predecessor of~$u$
  for some $r\in\rol(\Omc)$, and there are no other role connections.
  For any $u\in\Delta^{\Imc}$, we denote by $\ls_u$ the corresponding total
  preorder on~\Umc, that is, we define $\alpha\ls_u\beta$ iff
  $u\in\oc{\alpha\le\beta}^\Imc$, and by $\equiv_u$ the induced equivalence
  relation.
%!TEX encoding = UTF-8 Unicode

  As a first step in the construction of a \Go-model of \Omc, we define the
  auxiliary function $v\colon\Umc\times\Delta^\Imc\to[0,1]$ that satisfies the
  following conditions for all $u\in\Delta^\Imc$:
  \begin{enumerate}[label=(P\arabic*),leftmargin=*]
    \item\label{p1} for all $q\in\values$, we have $v(q,u)=q$,
    \item\label{p2} for all $\alpha,\beta\in\Umc$, we have
      $v(\alpha,u)\le v(\beta,u)$ iff $\alpha\ls_u\beta$,
    \item\label{p3} for all $\alpha\in\Umc$, we have
      $v(\inv(\alpha),u)=1-v(\alpha,u)$,
    \item\label{p4} if $u\neq\eps$, then for all $C\in\sub(\Omc)$ it holds that 
      $v(C,\parent{u})=v(\up{C},u)$.
  \end{enumerate}
  We define~$v$ by induction on the structure of $\Delta^\Imc$ starting 
  with~$\eps$. Let $\Umc/{\equiv_\eps}$ be the set of all equivalence
  classes of~$\equiv_\eps$. Then $\ls_\eps$ yields a total order $\le_\eps$
  on $\Umc/{\equiv_\eps}$. Since \Imc satisfies $\red(\Umc)$, we have
  $$[0]_\eps<_\eps[q_1]_\eps<_\eps\dots<_\eps[q_{k-1}]_\eps<_\eps[1]_\eps$$ \wrt
  this order. For every $[\alpha]_\epsilon\in\Umc/{\equiv_\eps}$, we now set
  $\inv([\alpha]_\eps):=[\inv(\alpha)]_\eps$. This function is well-defined by
  the axioms in~$\red(\Umc)$.
  On all $\alpha\in[q]_\eps$ for $q\in\values$, we now define
  $v(\alpha,\eps):=q$, which ensures that \ref{p1} holds.
  For the equivalence classes that do not contain a value from~\values, note
  that by~$\red(\Umc)$, every such class must be strictly between~$[q_i]_\eps$
  and~$[q_{i+1}]_\eps$ for $q_i,q_{i+1}\in\values$. We denote the $n_i$
  equivalence classes between $[q_i]_\eps$ and $[q_{i+1}]_\eps$ as follows:
  \[ [q_i]_\eps<_\eps E_1^i<_\eps\dots<_\eps E_{n_i}^i<_\eps[q_{i+1}]_\eps. \]
  For every $\alpha\in E_j^i$, we set
  $v(\alpha,\eps):=q_i+\tfrac{j}{n_i+1}(q_{i+1}-q_i)$, which ensures that
  \ref{p2} is also satisfied.
  Furthermore, observe that $1-q_{i+1}$ and $1-q_i$ are also adjacent 
  in~\values and we have
  \[[1-q_{i+1}]_\eps<_\eps\inv(E_{n_i}^i)<_\eps\dots<_\eps
    \inv(E_1^i)<_\eps[1-q_i]_\eps\]
  by the axioms in $\red(\Umc)$.
  Hence, it follows from the definition of~$v(\alpha,\eps)$ that \ref{p3}
  holds.

  Let now $u\in\Delta^\Imc$ be such that the function $v$, satisfying the properties 
  \ref{p1}--\ref{p4}, has already been
  defined for $\parent{u}$. Since \Imc is a tree model, there must be an $r\in\NR$ such
  that $(\parent{u},u)\in r^\Imc$.
  We again consider the set of equivalence classes $\Umc/{\equiv_u}$ and set
  $v(\alpha,u):=q$ for all $q\in\values$ and $\alpha\in[q]_u$, and
  $v(\alpha,u):=v(C,\parent{u})$ for all $C\in\sub(\Omc)$
  and $\alpha\in[\up{C}]_u$.
  To see that this is well-defined, consider the case that
  $[\up{C}]_u=[\up{D}]_u$, \ie $u\in\oc{\up{C}=\up{D}}^\Imc$. From the axioms 
  in
  $\red(\uparrow)$ and the fact that $(\parent{u},u)\in r^\Imc$, it follows
  that $\parent{u}\in\oc{C=D}^\Imc$, and thus 
  $[C]_{\parent{u}}=[D]_{\parent{u}}$.
  Since~\ref{p2} is satisfied for $\parent{u}$, we get
  $v(C,\parent{u})=v(D,\parent{u})$.
  The same argument shows that $[q]_u=[\up{q}]_u=[\up{C}]_u$ implies
  $v(q,\parent{u})=v(C,\parent{u})$.
  For the remaining equivalence classes, we can use a construction analogous to
  the case for~$\eps$ by considering the two unique neighboring equivalence
  classes that contain an element of $\values\cup\sub(\Omc)$ (for which
  $v$ has already been defined). This construction ensures that
  \ref{p1}--\ref{p4} hold for~$u$.

  Based on the function~$v$, we define the \Go-interpretation~$\Imc_\fuzzy$ over the domain
  $\Delta^{\Imc_\fuzzy}:=\Delta^\Imc$, where $a^{\Imc_\fuzzy}:=a^\Imc=\eps$;
  \begin{align*}
    A^{\Imc_\fuzzy}(u) &:= \begin{cases}
      v(A,u) &\text{if $A\in\sub(\Omc)$,} \\
      0 &\text{otherwise; \quad and}
    \end{cases} \\
    r^{\Imc_\fuzzy}(u,w) &:= \begin{cases}
      v(\lambda,w) &\text{if $(u,w)\in r^\Imc$,} \\
      0 &\text{otherwise.}
    \end{cases}
  \end{align*}
  We show by induction on the structure of~$C$ that
  \begin{equation}\label{If-correct}
    C^{\Imc_\fuzzy}(u)=v(C,u)
      \text{ for all }C\in\sub(\Omc)\text{ and }u\in\Delta^\Imc.
  \end{equation}
  For concept names, this holds by the definition of~$\Imc_\fuzzy$.
  For~$\top$, we know that $\top^{\Imc_\fuzzy}(u)=1=v(\top,u)$ by the
  definition of $\red(\top)$ and~\ref{p2}.
  For $\lnot C$, we have
  \[ (\lnot C)^{\Imc_\fuzzy}(u)=1-C^{\Imc_\fuzzy}(u)=1-v(C,u)=v(\lnot C,u) \]
  by the induction hypothesis and~\ref{p3}.
  For conjunctions $C\sqcap D$, we know that
  \begin{align*} 
  (C\sqcap D)^{\Imc_\fuzzy}(u) &
    {} = \min\{C^{\Imc_\fuzzy}(u),D^{\Imc_\fuzzy}(u)\} \\ &
    {} = \min\{v(C,u),v(D,u)\} \\ &
    {} = v(C\sqcap D,u) 
  \end{align*}
  by the definition of $\red(C\sqcap D)$ and~\ref{p2}. Implications can be 
  treated similarly.

  Consider a value restriction $\forall r.C\in\sub(\Omc)$.
  For every $w\in\Delta^\Imc$ with $(u,w)\in r^\Imc$, we have
  $w\in\oc{\up{\forall r.C}\le\lambda\ra C}^\Imc$ since \Imc satisfies
  $\red(\forall r.C)$. By the induction hypothesis, the fact that
  $\parent{w}=u$, \ref{p2}, and~\ref{p4}, this implies that
  $v(\forall r.C,u)
    \le v(\lambda,w)\ra v(C,w)
    = r^{\Imc_\fuzzy}(u,w)\ra C^{\Imc_\fuzzy}(w)$,
  and thus
  \begin{align*} (\forall r.C)^{\Imc_\fuzzy}(u)
    & {} = \inf\limits_{\mathclap{\phantom{======} w\in\Delta^{\Imc},\ (u,w)\in r^\Imc}} \quad
      r^{\Imc_\fuzzy}(u,w)\ra C^{\Imc_\fuzzy}(w) \\
    & {} \ge v(\forall r.C,u).
  \end{align*}
  Furthermore, by the existential restriction introduced in $\red(\forall r.C)$, we
  know that there exists a $w_0\in\Delta^\Imc$ such that $(u,w_0)\in r^\Imc$ 
  and
  $w_0\in\oc{\up{\forall r.C}\ge\lambda\ra C}^\Imc$.
  By the same arguments as above, we get
  \begin{align*}
  v(\forall r.C,u)
    {} & \ge r^{\Imc_\fuzzy}(u,w_0)\ra C^{\Imc_\fuzzy}(w_0) \\
    {} & \ge (\forall r.C)^{\Imc_\fuzzy}(u),
  \end{align*}
  which concludes the proof of~\eqref{If-correct} for~$\forall r.C$. As a
  by-product, we have found in the element~$w_0$ the witness required for satisfying the concept $\forall r.C$
  at~$u$.

  Consider now $\atLeast{n}{r}{C}\in\sub(\Omc)$.
  For any $n$\mbox{-}tuple $(w_1,\dots,w_n)$ of different domain elements with
  $(u,w_1),\dots,(u,w_n)\in r^\Imc$, by $\red(\atLeast{n}{r}{C})$
  there must be an index~$i$, $1\le i\le n$, such that
  $w_i\notin\oc{\up{\atLeast{n}{r}{C}}<\min\{\lambda,C\}}^\Imc$.
  Using arguments similar to those introduced above, we obtain that
  \begin{align*} 
  v(\atLeast{n}{r}{C},u) &
    {} \ge \min\{r^{\Imc_\fuzzy}(u,w_i),C^{\Imc_\fuzzy}(w_i) \\ &
    {} \ge \min_{j=1}^n\min\{r^{\Imc_\fuzzy}(u,w_j),C^{\Imc_\fuzzy}(w_j)\}. 
   \end{align*}
  On the other hand, we know that there are $n$ different elements
  $w_1^0,\dots,w_n^0\in\Delta^\Imc$ such that $(u,w_j^0)\in r^\Imc$ and
  $w_j\in\oc{\up{\atLeast{n}{r}{C}}\le\min\{\lambda,C\}}^\Imc$ for all~$j$,
  $1\le j\le n$. As in the case of~$\forall r.C$ above, we conclude that
  \begin{align*} 
  v(\atLeast{n}{r}{C},u) &
    {} \le \min_{j=1}^n\min\{r^{\Imc_\fuzzy}(u,w_j^0),C^{\Imc_\fuzzy}(w_j^0)\} \\ &
    {} \le (\atLeast{n}{r}{C})^{\Imc_\fuzzy}(u)
    \le v(\atLeast{n}{r}{C},u), 
  \end{align*}
  as required. Furthermore, $w_1^0,\dots,w_n^0$ are the required
  witnesses for $\atLeast{n}{r}{C}$ at~$u$.
%
%  Finally, we consider the case of at-most restrictions
%  $\atMost{n}{r}{C}\in\sub(\Omc)$.
%%
%  Since \Imc satisfies $\red(\atMost{n}{r}{C})$, we know that
%  $v(\atMost{n}{r}{C},u)\in\{0,1\}$ by~\ref{p1} and~\ref{p2}.
%%
%  Furthermore, this value is equal to~$1$ iff there are at most $n$ elements
%  $w\in\Delta^\Imc$ with $(u,w)\in r^\Imc$ and
%  $\min\{r^{\Imc_\fuzzy}(u,w),C^{\Imc_\fuzzy}(w)\}>0$. By
%  Proposition~\ref{prop:at-most-semantics}, this is exactly the semantics of
%  $(\atMost{n}{r}{C})^{\Imc_\fuzzy}(u)$.
%
  This concludes the proof of~\eqref{If-correct}.

  It remains to be shown that $\Imc_\fuzzy$ is a model of~\Omc. For every
  $\axiom{\assert{a}{C}\bowtie q}\in\Amc$, we have
  $a^\Imc=\eps\in[C\bowtie q]^\Imc$, and thus
  $C^{\Imc_\fuzzy}(a^{\Imc_\fuzzy})=v(C,\eps)\bowtie v(q,\eps)=q$
  by~\eqref{If-correct}, \ref{p1}, and~\ref{p2}. A similar argument works for handling
  order assertions of the form $\axiom{\assert{a}{C}\bowtie\assert{a}{D}}$.
  To conclude, consider an arbitrary GCI $\axiom{C\sqsubseteq D\ge q}\in\Tmc$ and
  $u\in\Delta^\Imc$. By the definition of $\red(\Tmc)$ and~\ref{p1}, we have
  $v(q,u)\le v(C,u)\ra v(D,u)$. Thus, \eqref{If-correct} and~\ref{p2} yield
  $C^{\Imc_\fuzzy}(u)\ra D^{\Imc_\fuzzy}(u)\ge q$. Thus, $\Imc_\fuzzy$ satisfies all the axioms 
  in \Omc, which concludes the proof.
%  \qed
\end{proof}
For the converse direction, we now show that it is possible to unravel every \Go-model of~\Omc
into a classical tree model of $\red(\Omc)$.

\begin{lemma}
\label{lem:completeness}
  If \Omc has a \Go-model, then $\red(\Omc)$ has a classical model.
\end{lemma}
\begin{proof}
%I think this is not needed anymore:
%  By the Löwenheim-Skolem-theorem, there must be a \Go-model~\Imc of~\Omc that
%  has a countable domain, and hence we can assume that
%  $\Delta^\Imc\subseteq\naturals$.
%
  Given a \Go-model~\Imc of~\Omc, we define a classical
  interpretation~$\Imc_\crisp$ over the domain $\Delta^{\Imc_\crisp}$ of all
  paths of the form $\path=r_1d_1\dots r_md_m$ with $r_i\in\NR$,
  $d_i\in\Delta^\Imc$, $m\ge 0$.
  We set $a^{\Imc_\crisp}:=\eps$ and
  \[ r^{\Imc_\crisp}:=\{(\path,\path rd) \mid
    \path\in\Delta^{\Imc_\crisp},\ d\in\Delta^\Imc\} \]
  for all $r\in\NR$.
  We denote by $\tail(r_1d_1\dots r_md_m)$ the element $d_m$ if $m>0$, and
  $a^\Imc$ if $m=0$. Similarly, we set $\prev(r_1d_1\dots r_md_m)$ to $d_{m-1}$
  if $m>1$, and to $a^\Imc$ if $m=1$. Finally, $\role(r_1d_1\dots r_md_m)$
  denotes $r_m$ whenever $m>0$.
  For any $\alpha\in\Umc$ and $\path\in\Delta^{\Imc_\crisp}$, we define $\alpha^\Imc(\path)$ as
  \begin{align*}
    C^\Imc(\tail(\path)) &\text{ if $\alpha=C\in\sub(\Omc)$;} \\
    C^\Imc(\prev(\path)) &\text{ if $\alpha=\up{C}, C\in\sub(\Omc)$;} \\
    q &\text{ if $\alpha=q\in\values$;} \\
    \role(\path)^\Imc(\prev(\path),\tail(\path)) &\text{\ if $\alpha=\lambda$;}\\
    1-\role(\path)^\Imc(\prev(\path),\tail(\path))
      &\text{\ if $\alpha=\lnot\lambda$.}
  \end{align*} 
  Note that for $\path=\eps$ this expression is only defined for
  $\alpha\in\values\cup\sub(\Omc)$. We fix the value of $\alpha^\Imc(\eps)$ for
  all other~$\alpha$ arbitrarily, in such a way that for all
  $\alpha,\beta\in\Umc$ we have $\alpha^\Imc(\eps)\le\beta^\Imc(\eps)$ iff
  $\inv(\beta)^\Imc(\eps)\le\inv(\alpha)^\Imc(\eps)$.
  We can now define the interpretation of all concept names 
  $\oc{\alpha\le\beta}$
  with $\alpha,\beta\in\Umc$ as
  \[ \oc{\alpha\le\beta}^{\Imc_\crisp}:=
    \{\path \mid \alpha^\Imc(\path)\le\beta^\Imc(\path)\}. \]
  It is easy to see that we have
  $\path\in\oc{\alpha\bowtie\beta}^{\Imc_\crisp}$ iff
  $\alpha^\Imc(\path)\bowtie\beta^\Imc(\path)$ also for all other order
  expressions~$\bowtie$, and that $\Imc_\crisp$ satisfies $\red(\Umc)$.
  We now show that $\Imc_\crisp$ satisfies the remaining parts of $\red(\Omc)$.

  For any order assertion $\axiom{\assert{a}{C}\bowtie\assert{a}{D}}\in\Amc$ we
  have $C^\Imc(a^\Imc)\bowtie D^\Imc(a^\Imc)$. This implies that
  $C^\Imc(\eps)\bowtie D^\Imc(\eps)$, and thus
  $a^\Imc=\eps\in\oc{C\bowtie D}^{\Imc_\crisp}$, as required. A similar argument
  works for assertions of the form $\axiom{\assert{a}{C}\bowtie q}$.
  Consider now a GCI $\axiom{C\sqsubseteq D\ge q}\in\Tmc$ and any
  $\path\in\Delta^{\Imc_\crisp}$. We know that
  $C^\Imc(\tail(\path))\ra D^\Imc(\tail(\path))\ge q$, and thus
  $\path\in\oc{q\le C\ra D}^{\Imc_\crisp}$.

  For $\red(\uparrow)$, consider any $\alpha,\beta\in\values\cup\sub(\Omc)$,
  $r\in\rol(\Omc)$, and $\path\in\oc{\alpha\bowtie\beta}^{\Imc_\crisp}$. Thus, 
  it
  holds that $\alpha^\Imc(\path)\bowtie\beta^\Imc(\path)$. Every $r$-successor
  of~$\path$ in~$\Imc_\crisp$ must be of the form $\path rd$. Since
  $\up{\alpha}^\Imc(\path rd)=\alpha^\Imc(\path)\bowtie\beta^\Imc(\path)=
    \up{\beta}^\Imc(\path rd)$, we know that all $r$-successors of~$\path$
  satisfy $\oc{\up{\alpha}\bowtie\up{\beta}}$.

  It remains to be shown that $\Imc_\crisp$ satisfies $\red(C)$ for all concepts
  $C\in\sub(\Omc)$. 
    For $C=\top$, the claim follows from the fact that
      $\top^\Imc(\path)=\top^\Imc(\tail(\path))=1$.
    For $\lnot C$, the result is trivial, and for conjunctions and
    implications, it follows from the semantics of $\sqcap$ and $\to$ and the
    properties of $\min$ and $\ra$, respectively.
      
    Consider the case of $\forall r.C$ and an arbitrary domain element
      $\path\in\Delta^{\Imc_\crisp}$, and set $d:=\tail(\path)$. Since \Imc is
      witnessed, there must be an $e\in\Delta^\Imc$ such that
      \begin{align*}
        \up{\forall r.C}^\Imc(\path re)
          & {} = (\forall r.C)^\Imc(d) \\
          & {} = r^\Imc(d,e)\ra C^\Imc(e) \\
          & {} = \lambda^\Imc(\path re)\ra C^\Imc(\path re).
      \end{align*}
      Since $(\path,\path re)\in r^{\Imc_\crisp}$, this shows that
      $\exists r.\oc{\up{\forall r.C}\ge\lambda\ra C}$ is satisfied
      by~$\path$ in~$\Imc_\crisp$.
      Additionally, for any $r$-successor~$\path re$ of~$\path$ we have
      \begin{align*}
        \up{\forall r.C}^\Imc(\path re)
          & {} = (\forall r.C)^\Imc(d) \\
          & {} \le r^\Imc(d,e)\ra C^\Imc(e) \\
          & {} = \lambda^\Imc(\path re)\ra C^\Imc(\path re),
      \end{align*}
      and thus $\forall r.\oc{\up{\forall r.C}\le\lambda\ra C}$ is also
      satisfied.
    
    For at-least restrictions $\atLeast{n}{r}{C}$, we similarly know that
      there are $n$ different elements $e_1,\dots,e_n$ such that, for all $i$,
      $1\le i\le n$,
      \begin{align*}
        \up{\atLeast{n}{r}{C}}^\Imc(\path re_i)
          & {} = (\atLeast{n}{r}{C})^\Imc(d) \\
          & {} = \min_{j=1}^n\min\{r^\Imc(d,e_j),C^\Imc(e_j)\} \\
          & {} \le \min\{r^\Imc(d,e_i),C^\Imc(e_i)\} \\
          & {} = \min\{\lambda^\Imc(\path re_i),C^\Imc(\path re_i)\}.
      \end{align*}
      Since also the elements $\path re_1$, \dots, $\path re_n$ are different,
      this shows that the at-least restriction
      $\atLeast{n}{r}{\oc{\up{\atLeast{n}{r}{C}}\le\min\{\lambda,C\}}}$ is
      satisfied by~$\Imc_\crisp$ at~$\path$.
      On the other hand, for \emph{all} $n$-tuples
      $(\path re_1,\dots,\path r e_n)$ of different $r$-successors of~$\path$
      and all $i$, $1\le i\le n$, we must have
      \begin{align*}
        \up{\atLeast{n}{r}{C}}^\Imc(\path re_i)
          & {} = (\atLeast{n}{r}{C})^\Imc(d)  \\
          & {} \ge \min_{j=1}^n\min\{r^\Imc(d,e_j),C^\Imc(e_j)\} \\
          & {} = \min_{j=1}^n\min\{\lambda^\Imc(\path re_j),C^\Imc(\path re_j)\},
      \end{align*}
      and thus there must be at least one $j$, $1\le j\le n$, such that
      \[ \path re_j\in
        \oc{\up{\atLeast{n}{r}{C}}\ge\min\{\lambda,C\}}^{\Imc_\crisp}. \]
      In other words, there can be no $n$ different elements of the form
      $\path re$ that satisfy $\path re\in
        \oc{\up{\atLeast{n}{r}{C}}<\min\{\lambda,C\}}^{\Imc_\crisp}$,
      \ie $\path\notin
        \atLeast{n}{r}%
          {\oc{\up{\atLeast{n}{r}{C}}<\min\{\lambda,C\}}}^{\Imc_\crisp}$.
%
%    For at-most restrictions, we have
%      $(\atMost{n}{r}{C})^\Imc(d)\in\{0,1\}$ for all $d\in\Delta^\Imc$ by
%      Proposition~\ref{prop:at-most-semantics}. Hence, the first GCI in the
%      definition of $\red(\atMost{n}{r}{C})$ is satisfied.
%%
%      For the second one, the same proposition yields that
%      $(\atMost{n}{r}{C})^\Imc(d)=1$ iff there are at most $n$ elements
%      $e\in\Delta^\Imc$ with $\min\{r^\Imc(d,e),C^\Imc(e)\}>0$. Using arguments
%      that should be familiar by now, this is equivalent to the semantics of
%      the classical axiom
%      $\oc{1\le(\atMost{n}{r}{C}}\equiv\atMost{n}{r}{\oc{0<\min\{\lambda,C\}}}$
%      under~$\Imc_\crisp$.
% \qed
\end{proof}
In contrast to the reductions for finitely valued Gödel
FDLs~\cite{BDGS-IJAR09,BDGS-IJUF12}, the size of $\red(\Omc)$ is always
polynomial in the size of~\Omc.
The reason is that we do not translate the concepts occurring in the ontology
recursively, but rather introduce a polynomial-sized subontology $\red(C)$ for
each relevant subconcept~$C$.
Moreover, we do not need to introduce role hierarchies for our reduction, since
the value of role connections is expressed using the special element~$\lambda$.
\ExpTime-completeness of concept satisfiability in classical
\ALCQ~\cite{Schi-IJCAI91,Tobi-01} now yields the following result.

\begin{theorem}
\label{thm:local-consistency}
  Local consistency in \Go-\IALCQ is \ExpTime-complete.
\end{theorem}

\section{Conclusions}

Using a combination of techniques developed for infinitely valued Gödel
extensions of \ALC~\cite{BoDP-KR14} and for finitely valued Gödel extensions
of \SROIQ~\cite{BDGS-IJAR09,BDGS-IJUF12}, we have shown that local consistency
in infinitely valued \Go-\IALCQ is \ExpTime-complete. Our reduction is more
practical than the automata-based approach proposed by~\citeauthor{BoDP-KR14}~\shortcite{BoDP-KR14} 
and does not exhibit the exponential blowup of the reductions 
developed by~\citeauthor{BDGS-IJAR09}~\shortcite{BDGS-IJAR09,BDGS-IJUF12}.
Beyond the complexity results, an important benefit of our approach is that it does 
not need the development of a specialized fuzzy DL reasoner, but can use any
state-of-the-art reasoner for classical \ALCQ without modifications.
For that reason, this new reduction aids to shorten the gap between efficient classical and fuzzy DL reasoners.

In future work, we want to extend this result to full consistency, possibly
using the notion of a \emph{pre-completion} as introduced in~\cite{BoDP-KR14}. Our
ultimate goal is to provide methods for reasoning efficiently in infinitely valued Gödel extensions of the
very expressive DL \SROIQ, underlying OWL\,2\,DL.
We believe that it is possible to treat
transitive roles, inverse roles, role hierarchies, and nominals using the
extensions of the automata-based approach developed originally for finitely valued FDLs
in~\cite{BoPe-JoDS13,BoPe-14,Borg-DL14}.

As done previously in~\cite{BDGS-IJUF12}, we can also combine our reduction with the
one for infinitely-valued Zadeh semantics. Although Zadeh semantics is not based on t-norms, 
it nevertheless is important to handle it correctly, as it is one of the most widely used semantics for fuzzy
applications. It also has some properties that make it closer to the classical semantics, and hence become
a natural choice for simple applications.

A different direction for future research would be to integrate our reduction
directly into a classical tableaux reasoner. Observe that the definition of $\red(C)$ is
already very close to the rules employed in (classical and fuzzy) tableaux
algorithms (see, e.g.~\cite{BaSa-SL01,BoSt-FSS09}). However, the tableaux procedure would
need to deal with total preorders in each node, possibly using an external
solver.

\section*{Acknowledgements}
This work was partially supported by the German Research Foundation (DFG) under the grant
BA 1122/17-1 (FuzzyDL) and within the Cluster of Excellence `cfAED.' Most of this work was developed
while R.\ Pe\~naloza was still affiliated with TU Dresden and the Center for Advancing Electronics Dresden,
Germany.

\bibliographystyle{named}
\bibliography{medium-string,GoedelSROIQ}

\begin{thebibliography}{}

\bibitem[\protect\citeauthoryear{Baader and Pe{\~n}aloza}{2011}]{BaPe-FroCoS11}
Franz Baader and Rafael Pe{\~n}aloza.
\newblock On the undecidability of fuzzy description logics with {GCI}s and
  product t-norm.
\newblock In {\em Proc.\ of the 8th Int.\ Symp.\ on Frontiers of Combining
  Systems (FroCoS'11)}, volume 6989 of {\em Lecture Notes in Computer Science},
  pages 55--70. Springer-Verlag, 2011.

\bibitem[\protect\citeauthoryear{Baader and Sattler}{2001}]{BaSa-SL01}
Franz Baader and Ulrike Sattler.
\newblock An overview of tableau algorithms for description logics.
\newblock {\em Studia Logica}, 69(1):5--40, 2001.

\bibitem[\protect\citeauthoryear{Baader \bgroup \em et al.\egroup
  }{2007}]{BCM+07}
Franz Baader, Diego Calvanese, Deborah~L. McGuinness, Daniele Nardi, and
  Peter~F. Patel-Schneider, editors.
\newblock {\em The Description Logic Handbook: Theory, Implementation, and
  Applications}.
\newblock Cambridge University Press, 2nd edition, 2007.

\bibitem[\protect\citeauthoryear{Bobillo and Straccia}{2009}]{BoSt-FSS09}
Fernando Bobillo and Umberto Straccia.
\newblock Fuzzy description logics with general t-norms and datatypes.
\newblock {\em Fuzzy Sets and Systems}, 160(23):3382--3402, 2009.

\bibitem[\protect\citeauthoryear{Bobillo and Straccia}{2011}]{BoSt-ISC11}
Fernando Bobillo and Umberto Straccia.
\newblock Reasoning with the finitely many-valued {{\L}}ukasiewicz fuzzy
  description logic {$\mathcal{SROIQ}$}.
\newblock {\em Information Sciences}, 181:758--778, 2011.

\bibitem[\protect\citeauthoryear{Bobillo and Straccia}{2013}]{BoSt-13}
Fernando Bobillo and Umberto Straccia.
\newblock Finite fuzzy description logics and crisp representations.
\newblock In {\em Uncertainty Reasoning for the Semantic Web {II}}, volume 7123
  of {\em Lecture Notes in Computer Science}, pages 102--121. Springer-Verlag,
  2013.

\bibitem[\protect\citeauthoryear{Bobillo \bgroup \em et al.\egroup
  }{2009}]{BDGS-IJAR09}
Fernando Bobillo, Miguel Delgado, Juan G{\'o}mez-Romero, and Umberto Straccia.
\newblock Fuzzy description logics under {G}{\"o}del semantics.
\newblock {\em International Journal of Approximate Reasoning}, 50(3):494--514,
  2009.

\bibitem[\protect\citeauthoryear{Bobillo \bgroup \em et al.\egroup
  }{2012}]{BDGS-IJUF12}
Fernando Bobillo, Miguel Delgado, Juan G{\'o}mez-Romero, and Umberto Straccia.
\newblock Joining {G}{\"o}del and {Z}adeh fuzzy logics in fuzzy description
  logics.
\newblock {\em International Journal of Uncertainty, Fuzziness and
  Knowledge-Based Systems}, 20(4):475--508, 2012.

\bibitem[\protect\citeauthoryear{Borgwardt and
  Pe{\~n}aloza}{2013}]{BoPe-JoDS13}
Stefan Borgwardt and Rafael Pe{\~n}aloza.
\newblock The complexity of lattice-based fuzzy description logics.
\newblock {\em Journal on Data Semantics}, 2(1):1--19, 2013.

\bibitem[\protect\citeauthoryear{Borgwardt and Pe{\~n}aloza}{2014}]{BoPe-14}
Stefan Borgwardt and Rafael Pe{\~n}aloza.
\newblock Finite lattices do not make reasoning in {$\mathcal{ALCOI}$} harder.
\newblock In {\em Uncertainty Reasoning for the Semantic Web III}, volume 8816
  of {\em Lecture Notes in Artificial Intelligence}, pages 122--141.
  Springer-Verlag, 2014.

\bibitem[\protect\citeauthoryear{Borgwardt and
  Pe{\~n}aloza}{2015}]{BoPe-FroCoS15}
Stefan Borgwardt and Rafael Pe{\~n}aloza.
\newblock Reasoning in expressive description logics under infinitely valued
  g{\"o}del semantics.
\newblock In {\em Proc. of FroCoS 2015}, Lecture Notes in Computer Science.
  Springer-Verlag, 2015.
\newblock To appear.

\bibitem[\protect\citeauthoryear{Borgwardt \bgroup \em et al.\egroup
  }{2014}]{BoDP-KR14}
Stefan Borgwardt, Felix Distel, and Rafael Pe{\~n}aloza.
\newblock Decidable {G}{\"o}del description logics without the finitely-valued
  model property.
\newblock In {\em Proc.\ of the 14th Int.\ Conf.\ on Principles of Knowledge
  Representation and Reasoning (KR'14)}, pages 228--237. {AAAI} Press, 2014.

\bibitem[\protect\citeauthoryear{Borgwardt \bgroup \em et al.\egroup
  }{2015}]{BoDP-AI15}
Stefan Borgwardt, Felix Distel, and Rafael Pe{\~n}aloza.
\newblock The limits of decidability in fuzzy description logics with general
  concept inclusions.
\newblock {\em Artificial Intelligence}, 218:23--55, 2015.

\bibitem[\protect\citeauthoryear{Borgwardt}{2014}]{Borg-DL14}
Stefan Borgwardt.
\newblock Fuzzy {DL}s over finite lattices with nominals.
\newblock In {\em Proc.\ of the 27th Int.\ Workshop on Description Logics
  (DL'14)}, volume 1193 of {\em CEUR Workshop Proceedings}, pages 58--70, 2014.

\bibitem[\protect\citeauthoryear{Cerami and Straccia}{2013}]{CeSt-ISc13}
Marco Cerami and Umberto Straccia.
\newblock On the (un)decidability of fuzzy description logics under
  {{\L}}ukasiewicz t-norm.
\newblock {\em Information Sciences}, 227:1--21, 2013.

\bibitem[\protect\citeauthoryear{H{\'a}jek}{2001}]{Haje-01}
Petr H{\'a}jek.
\newblock {\em Metamathematics of Fuzzy Logic (Trends in Logic)}.
\newblock Springer-Verlag, 2001.

\bibitem[\protect\citeauthoryear{H{\'a}jek}{2005}]{Haje-FSS05}
Petr H{\'a}jek.
\newblock Making fuzzy description logic more general.
\newblock {\em Fuzzy Sets and Systems}, 154(1):1--15, 2005.

\bibitem[\protect\citeauthoryear{Kazakov}{2004}]{Kaza-JELIA04}
Yevgeny Kazakov.
\newblock A polynomial translation from the two-variable guarded fragment with
  number restrictions to the guarded fragment.
\newblock In {\em Proc.\ of the 9th Eur.\ Conf.\ on Logics in Artificial
  Intelligence (JELIA'04)}, volume 3229 of {\em Lecture Notes in Computer
  Science}, pages 372--384. Springer-Verlag, 2004.

\bibitem[\protect\citeauthoryear{Klement \bgroup \em et al.\egroup
  }{2000}]{KlMP-00}
Erich~Peter Klement, Radko Mesiar, and Endre Pap.
\newblock {\em Triangular Norms}.
\newblock Trends in Logic, Studia Logica Library. Springer-Verlag, 2000.

\bibitem[\protect\citeauthoryear{Schild}{1991}]{Schi-IJCAI91}
Klaus Schild.
\newblock A correspondence theory for terminological logics: Preliminary
  report.
\newblock In {\em Proc.\ of the 12th Int.\ Joint Conf.\ on Artificial
  Intelligence (IJCAI'91)}, pages 466--471. Morgan Kaufmann, 1991.

\bibitem[\protect\citeauthoryear{Schmidt-Schau{\ss} and
  Smolka}{1991}]{ScSm-AI91}
Manfred Schmidt-Schau{\ss} and Gert Smolka.
\newblock Attributive concept descriptions with complements.
\newblock {\em Artificial Intelligence}, 48(1):1--26, 1991.

\bibitem[\protect\citeauthoryear{Straccia}{2001}]{Stra-JAIR01}
Umberto Straccia.
\newblock Reasoning within fuzzy description logics.
\newblock {\em Journal of Artificial Intelligence Research}, 14:137--166, 2001.

\bibitem[\protect\citeauthoryear{Straccia}{2004}]{Stra-JELIA04}
Umberto Straccia.
\newblock Transforming fuzzy description logics into classical description
  logics.
\newblock In {\em Proc.\ of the 9th Eur.\ Conf.\ on Logics in Artificial
  Intelligence (JELIA'04)}, volume 3229 of {\em Lecture Notes in Computer
  Science}, pages 385--399. Springer-Verlag, 2004.

\bibitem[\protect\citeauthoryear{Tobies}{2001}]{Tobi-01}
Stephan Tobies.
\newblock {\em Complexity Results and Practical Algorithms for Logics in
  Knowledge Representation}.
\newblock PhD thesis, RWTH Aachen, Germany, 2001.

\end{thebibliography}

\end{document}